\documentclass[11pt]{article}

\usepackage{latexsym}
\usepackage{graphics}
\usepackage{amsmath}
\usepackage{xspace}
\usepackage{amssymb}
\usepackage{psfrag}
\usepackage{epsfig}
\usepackage{amsthm}
\usepackage{pst-all}

\usepackage{color}

\definecolor{Red}{rgb}{1,0,0}
\definecolor{Blue}{rgb}{0,0,1}
\definecolor{Olive}{rgb}{0.41,0.55,0.13}
\definecolor{Yarok}{rgb}{0,0.5,0}
\definecolor{Green}{rgb}{0,1,0}
\definecolor{MGreen}{rgb}{0,0.8,0}
\definecolor{DGreen}{rgb}{0,0.55,0}
\definecolor{Yellow}{rgb}{1,1,0}
\definecolor{Cyan}{rgb}{0,1,1}
\definecolor{Magenta}{rgb}{1,0,1}
\definecolor{Orange}{rgb}{1,.5,0}
\definecolor{Violet}{rgb}{.5,0,.5}
\definecolor{Purple}{rgb}{.75,0,.25}
\definecolor{Brown}{rgb}{.75,.5,.25}
\definecolor{Grey}{rgb}{.5,.5,.5}

\newcommand{\R}{\mathbb{R}}

\newcommand{\ignore}[1]{\relax}

\newtheorem{theorem}{Theorem}[section]

\newtheorem{lemma}[theorem]{Lemma}

\definecolor{Red}{rgb}{1,0,0}
\definecolor{Blue}{rgb}{0,0,1}
\definecolor{Olive}{rgb}{0.41,0.55,0.13}
\definecolor{Green}{rgb}{0,1,0}
\definecolor{MGreen}{rgb}{0,0.8,0}
\definecolor{DGreen}{rgb}{0,0.55,0}
\definecolor{Yellow}{rgb}{1,1,0}
\definecolor{Cyan}{rgb}{0,1,1}
\definecolor{Magenta}{rgb}{1,0,1}
\definecolor{Orange}{rgb}{1,.5,0}
\definecolor{Violet}{rgb}{.5,0,.5}
\definecolor{Purple}{rgb}{.75,0,.25}
\definecolor{Brown}{rgb}{.75,.5,.25}
\definecolor{Grey}{rgb}{.5,.5,.5}
\definecolor{Pink}{rgb}{1,0,1}
\definecolor{DBrown}{rgb}{.5,.34,.16}
\definecolor{Black}{rgb}{0,0,0}

\author{
{\sf David Gamarnik}\thanks{MIT; e-mail: {\tt gamarnik@mit.edu}.Research supported  by the NSF grants CMMI-1335155.}
\and
{\sf Sidhant Misra}\thanks{Los Alamos Research Laboratory; e-mail: {\tt sidhant@mit.edu}}
}

\begin{document}

\title{A Note on Alternating Minimization Algorithm for the Matrix Completion Problem}
\date{}

\maketitle

\begin{abstract}
We consider the problem of reconstructing a low rank matrix from a subset of its entries and analyze two
variants of the so-called Alternating Minimization algorithm, which has been proposed in the past. We establish that
when the underlying matrix has rank $r=1$, has positive bounded entries, and the graph $\mathcal{G}$  underlying the revealed entries has
bounded degree and diameter which is at most logarithmic
in the size of the matrix, both algorithms succeed in reconstructing the matrix approximately in polynomial time starting
from an arbitrary initialization. We further provide simulation results which
suggest that the second algorithm which is based on the message passing type updates, performs significantly better.
\end{abstract}

\section{Introduction}
Matrix completion refers to the problem of recovering a low rank matrix from an incomplete subset of its entries. This problem arises in a vast number of applications that involve \emph{collaborative filtering}, where one attempts to predict the unknown preferences of a certain user based on collective known preferences of a large number of users. It attracted a lot of attention in recent times due to its application in recommendation systems
and the well-known Netflix Prize.

\subsection{Formulation}
Let $M  = \alpha \beta^T$ be a rank $r$ matrix, where $\alpha \in \mathbb{R}^{m\times r}, \beta \in \mathbb{R}^{n \times r}$.
Let $\mathcal{E} \subset [m]\times[n]$ be a subset of the indices
and let $M_{\mathcal{E}}$ denote the entries of $M$ corresponding to the subset $\mathcal{E}$.
The matrix completion problem is the problem of reconstructing $M$ using efficient (polynomial time) algorithms given $M_{\mathcal{E}}$.

Without any further assumptions, the matrix completion problem is NP-hard \cite{Meka08}.
However under certain conditions, the problem has been shown to be tractable. The most common assumption typically considered in the literature is that the matrix
$M$ is ``incoherent" and the subset $\mathcal{E}$ is chosen using some random mechanism, for example uniformly at random.
The incoherence condition was introduced Candes and Recht~\cite{candes2009exact} and Candes and Tao~\cite{CandesTao09}, where it was shown that convex relaxation
resulting in nuclear norm minimization succeeds in reconstructing the matrix,
assuming a certain lower bound on the size of $\mathcal{E}$.
Keshevan et al.~\cite{keshavan2009matrix} and \cite{keshavan2009matrix2}, use an algorithm consisting of a truncated singular value projection followed by a local minimization subroutine
on the Grassmann manifold and show that it succeeds when $|\mathcal{E}|  = \Omega(n r \log n)$. Jain et al.~\cite{jain2013low} show that the local minimization
in~\cite{keshavan2009matrix} can be successfully replaced by Alternating Minimization algorithm. The use of Belief Propagation for matrix factorization has also been studied by physicists in \cite{Lenka14} heuristically.
This is just a small subset of a vast literature on matrix completion problem which is the most relevant to the present work.

\subsection{Algorithms and the results}
For the rest of the paper we will assume for simplicity that $m =n$, though our results easily extend to the more general case $m = \Theta(n)$.
Let $\mathcal{V}_R$ and $\mathcal{V}_C$ denote the sets of rows and columns of $M$ respectively, both indexed for simplicity
by sets $\{1,2,\ldots,n\}=[n]$, and
let $\mathcal{G}$ be a bipartite undirected graph on the vertex set $\mathcal{V} = \mathcal{V}_R \cup \mathcal{V}_C$  with edge set $\mathcal{E}$, where we recall
that $\mathcal{E}$ is the set of revealed entries of $M$. Specifically, the edge $(i,j)\in\mathcal{E}$ if and only if the entry $M_{i,j}$ is revealed.
Denote by $\Delta = \Delta(n)$ the maximum degree of $\mathcal{G}$ (the maximum number of neighbors among all nodes of $\mathcal{G}$).
The graph $\mathcal{G}$ represents
the structure of the revealed entries of $M$. We denote the $i$th row of $\alpha$ by $\alpha_i$ and the $j$th row of $\beta$ by $\beta_j$,
both thought of as column vectors.
Then $M_{ij} = \alpha_i^T \beta_j$.
The matrix completion is then the problem of finding $X,Y\in \R^{n\times r}$, such that for every $(i,j)\in\mathcal{E}$,
$X_i^TY_j=M_{i,j}=\alpha_i^T \beta_j$, where $(X_i, 1\le i\le n)$ and $(Y_j, 1\le j\le n)$ are rows of $X$ and $Y$ respectively. Alternatively, one can
consider the optimization problem
\begin{align} \label{opt}
	\min_{X, Y \in \R^{n \times r}} \sum_{(i,j) \in \mathcal{E}} |X_i^T Y_j - M_{ij} |^2,
\end{align}
and seek solutions with zero objective value.
Observe that, computational tractability aside one can only obtain $X=\alpha$ and $Y=\beta$ up to orthogonal transformation, since of any orthogonal $r\times r$
matrix $\Gamma$, $\tilde\alpha\triangleq \alpha\Gamma, \tilde\beta\triangleq \beta\Gamma$ solve the same matrix completion problem.

In this paper  we consider two versions of the Alternating Minimization algorithm which we call Vertex Least Square (VLS) and Edge Least Square (ELS) algorithms,
where VLS is identical to the Alternating Minimization algorithm analyzed in~\cite{jain2013low}, and ELS is a message passing type version of the VLS,
which has some similarity with the widely studied Belief Propagation algorithm.
Unlike~\cite{jain2013low}, where VLS was used as a local optimization subroutine following a singular value projection, i.e., with a warm start, we consider
the issue of global convergence of VLS and ELS. For the special case of rank $r=1$, bounded positive entries of $M$, bounded degree of $\mathcal{G}$,
and when the diameter of $\mathcal{G}$ is
$O(\log n)$, we establish that both algorithms converge to the correct factorization of $M$ geometrically fast. In particular,
the algorithms produce rank $r=1$ $\epsilon$-approximation of $M$ in time $O(n^{\gamma}\log n)$, where $\gamma$ only depends on the parameters of the model.
Our proof approach is based on establishing a certain contraction for the steps of VLS and ELS similar to the ones used for bounding mixing rates of Markov chain.

Even though our theoretical results show similar performance for VLS and ELS algorithms, experimentally we show that the ELS performs better and often significantly better.
Specifically, in Section~\ref{sec:sim} we show that for certain classes of randomly constructed matrices ELS converges much faster than VLS, and for other
classes of random graphs, ELS is converging while VLS is not. At this stage the theoretical understanding of such an apparent difference in the performance
of two algorithms is lacking and constitutes an interesting open challenge, especially in light of the fact that VLS was a major component in the award winning algorithms for the Netflix challenge \cite{Korin09}, \cite{Korensolo09}.

\section{VLS and ELS Algorithms} \label{sec:LS}
We now introduce and analyze two iterative decentralized algorithms based on the Alternating Minimization principle, that attempt to solve the non-convex least squares problem in (\ref{opt}). The first is what we call the Vertex Least Squares (VLS) algorithm. For the bi-partite graph $\mathcal{G}$ we write $i\sim j$ if $i\in\mathcal{V}_R$ is connected to $j\in\mathcal{G}_C$.

{\rule{\linewidth}{1pt}}
\vspace{-0.09in}

VERTEX LEAST SQUARES (VLS)
\vspace{-0.07in}

{\rule{\linewidth}{1pt}}
\begin{enumerate}
	\item Initialization: $x_{i,0}, y_{j,0}$. Matrix $M$ and graph $\mathcal{G}$.
	\item For $t = 1$ to $T$:
		For each $i \in \mathcal{V}_R$, set
		\begin{align}			x_{i, t+1} = \arg \min_{x \in \mathbb{R}^r} \sum_{j: \mathcal{V}_C\ni j \sim i} \left(x^Ty_{j,t} -  M_{ij}\right)^2.  \label{VLSupdate1} \end{align}
		For each $j \in \mathcal{V}_C$, set
		\begin{align}		y_{j,t+1} = \arg \min_{y \in \mathbb{R}^r} \sum_{i: \mathcal{V}_R\ni i \sim j} \left(y^Tx_{i,t+1} - M_{ij}\right)^2.	\label{VLSupdate2} \end{align}	
	\item Set $X_T=(x_{i,T}, 1\le i\le n), Y_T=(y_{j,T}, 1\le j\le n)$.   Output $\hat{M} = X_T Y_T^T$.			
\end{enumerate}

\rule{\linewidth}{1pt}

Each iteration of the VLS consists of solving $2  n$ least squares problems, so the total computation time per iteration is $O(r^2 \Delta n )$.

The VLS algorithm in the above form is identical to Alternating Minimization \cite{jain2013low} and exploits the biconvex structure of the objective in (\ref{opt}). We prefer to write the iterations of this algorithm in
the above form
to highlight the local decentralized nature of the updates at each vertex. In \cite{jain2013low}, this algorithm was used as a local minimization subroutine with a warm start provided by an SVD projection step prior to it. As we are about to establish, VLS solves the matrix completion problem for the case $r=1$ under non-negativity assumption, without any warm starts.
Furthermore, we will present simulation results showing that in many cases VLS solves the completion problem for general $r$.

Our main result concerning the $r=1$ is as follows. For every matrix $A=(a_{i,j}, 1\le i,j\le n)$, we denote by $\|A\|_F$
is its Frobenius norm: $\|A\|_F=\left(\sum_{i,j}a_{ij}^2\right)^{1\over 2}$.

\begin{theorem} \label{theorem:VLS}
	Let $M = \alpha \beta^T$ with $\alpha, \beta \in \mathbb{R}^n$ and suppose there exists $0 < b < 1$ such that for all $i,j \in [n]$, we have
	$b \leq \alpha_i, \beta_j \leq 1/b$. Suppose that the graph $\mathcal{G}$ is connected and has diameter $d \le c \log n$ for some fixed constant $c$,
and maximum degree $\Delta$.
	 There exists a constant $\alpha > 0$ which depends on $c,\Delta$ and $b$ only,
 such that for any initialization  $b \leq x_{i,0},y_{i,0} \leq 1/b, i \in [n]$ and
$\epsilon >0$, there exists an iteration number
	$T = O(n^\alpha \log n)$ such that after $T$ iterates of VLS, we have $\frac{1}{n} \|X_T Y_T^T - M \|_F < \epsilon$.
\end{theorem}
Before proceeding to the proof of Theorem \ref{theorem:VLS}, we remark here that in \cite{jain2013low}, the success of VLS was established by showing that the VLS updates resemble a power method with bounded error.
In our proof we also show that VLS updates are like  time varying power method updates, but without any error term. In \cite{jain2013low}, the warm start VLS required that principal angle distance between the left
and right singular vector subspaces of the actual matrix $M$ and the initial iterates are at most $0.5$. With the conditions given in Theorem \ref{theorem:VLS}, this may not always be the case.
From \cite{jain2013low}, the subspace distance between two vectors (rank $1$ case) is given by
\begin{align}
	d(u,v) = 1 - \left(\frac{u^T v}{\|u\| \|v\|}\right)^2.
\end{align}
Suppose that
\begin{align}
	\alpha_i = \begin{cases}
		b, \ &1 \leq i \leq n/2, \\
		1/b, \ &n/2 + 1 \le i \leq n. \\
	\end{cases}
	x_{i,0} = \begin{cases}
		1/b, \ &1 \leq i \leq n/2, \\
		b, \ &n/2 + 1 \le i \leq n.
	\end{cases}
\end{align}
Then $d(x, \alpha ) = 1 - \frac{4b^4}{(1  +b^4)^2}$ is greater than $1/2$ when $b$ is a small constant. In fact the subspace distance can be very close to one.
 Nevertheless, according to Theorem \ref{theorem:VLS} VLS converges
to the correct solution.

\begin{proof}[Proof of Theorem \ref{theorem:VLS}]
Fix $\epsilon>0$ and find $\delta>0$ small enough so that
\begin{align}\label{eq:epsilon-delta}
{2\delta \over b-\delta}<b^2\epsilon.
\end{align}

	From the update rules for VLS in Eq.~(\ref{VLSupdate1})-(\ref{VLSupdate2}), we can write
	\begin{align} \label{eq:VLSclosed1}
		x_{i, t+1} = \frac{\sum_{j: j \sim i}   M_{ij} y_{j,t}}{\sum_{j: j \sim i} y_{j,t}^2} \quad \mbox{and} \quad y_{j,t+1} = \frac{\sum_{i: i \sim j}
M_{ij} x_{i,t+1} }{\sum_{i: i \sim j} x_{i,t+1}^2}.
	\end{align}
	Let $A$ be the  adjacency matrix of $\mathcal{G}$, i.e., $A_{ij}  = \begin{cases}  1, \ &\mbox{if } i \sim j \\ 0, \ &\mbox{otherwise} \end{cases}$. \\

	Define $u_{i,t} = \frac{x_{i,t}}{\alpha_i}$ and $v_{j,t} = \frac{y_{j,t}}{\beta_j}$. With the chosen initial conditions in the theorem, we have that $b^2 \leq u_{i,0}, v_{j,0} \leq 1/b^2$. Using (\ref{eq:VLSclosed1}), the updates
	for $u_{i,t}$ and $v_{j,t}$ can be written as,
	\begin{align} \label{eq:VLSclosed2}
		u_{i,t+1} = \sum_{j: j \sim i} \frac{     y_{j,t}^2   }{\sum_{k: k \sim i} y_{k,t}^2} \frac{1}{v_{j,t}} \quad \mbox{and} \quad
		\frac{1}{v_{j,t}} = \sum_{i: j \sim i} \frac{   \alpha_i x_{i,t} }{\sum_{k: j \sim k} \alpha_i x_{i,t}} u_{i,t}.
	\end{align}
	
	The convex combination update rules in (\ref{eq:VLSclosed2}) imply that all future iterates satisfy $b^2 \leq u_{i,t}, v_{j,t} \leq 1/b^2$ and $b^3 \leq x_{i,t}, y_{j,t} \leq 1/b^3$.
	Combining the two updates in (\ref{eq:VLSclosed2}), we see
	that $u_{i,t+1}$ can be expressed as a convex combination of $u_{i,t}$, i.e., there exists a stochastic matrix $P_t$ such that $u_{t+1} = P_t u_t$, where $u_t = (u_{i,t}, \ i \in [n])$ expressed as a column vector.
	It is apparent that the support of $P_t$ is same as the support of $AA^T$, i.e., $P_t$ is the transition probability matrix of a random walk on $(\mathcal{V}_R, \mathcal{E}_R)$, where $(i_1, i_2) \in \mathcal{E}_R$
	if and only if $i_2$ is a distance two neighbor of $i_1$ in $\mathcal{G}$. Although $P_t$ depends on $t$, we can prove some useful properties satisfied by $P_t$ that hold for all times $t$. In particular observe from (\ref{eq:VLSclosed2}) that since $x_{i,t}$ and $y_{i,t}$ are bounded by $b^3$ and $1/b^3$, then each non-zero entry of $P_t$
is bounded below by $z\triangleq b^6/\Delta$.

\ignore{	
	\begin{lemma} \label{Puniformbound}
There exists  $0 < z < 1$ that depends on $b$ and $\Delta$ only such that
		all non-zero entries of matrix $P_t$ satisfy $P_{t,ij}\ge z$ for all $t$.
	\end{lemma}
}	

Recalling that $d$ stands for the diameter of  $\mathcal{G}$, define the sequence of matrices $\{ Q_k\}_{k=1}^{\infty}$ as
	\begin{align}
		Q_k = \prod_{t= (k-1)d + 1}^{kd} P_t.
 	\end{align}
	Then for any $k$ and $i,j\in [n]$,  $Q_k$ satisfies $Q_{k, ij} \geq z^d = z^{c \log n} \triangleq \frac{1}{n^\alpha}$, where $\alpha = - c \log z > 0$.
	 Let $w_t = u_{dt}$. Then,
	 \begin{align*}
	 	\max_{1\le i\le n} w_{t+1,i} &\leq \max_{1\le i\le n} w_{t,i} (1 - n^{ - \alpha}) + \min_i w_{t,i} n^{-\alpha}, \\
		\min_{1\le i\le n} w_{t+1,i} &\geq \min_{1\le i\le n} w_{t,i} (1 - n^{ - \alpha}) + \max_i w_{t,i} n^{-\alpha}.
	 \end{align*}
	 Combining the above gives
	 \begin{align}
	 	\left(\max_{1\le i\le n} w_{t+1,i}  - \min_{1\le i\le n} w_{t+1,i}\right)  \leq \left(1 - 2 n^{-\alpha}\right)
\left(\max_{1\le i\le n} w_{t,i} - \min_{1\le i\le n} w_{t,i}\right). \label{LR_contraction}
	 \end{align}
	For $t = \gamma n^{\alpha}$ and for large enough $\gamma$, we get by applying (\ref{LR_contraction}) recursively,
	\begin{align}
		\left(\max_{1\le i\le n} w_{t,i}  - \min_{1\le i\le n} w_{t,i}\right) \leq \delta.
	\end{align}
	Substituting the definition of $w_{t}$, we get $(\max_{1\le i\le n} u_{T,i}  - \min_{1\le i\le n} u_{T,i}) \leq \delta$, where $T = \gamma c n^{\alpha} \log n$.
	This means there exists a constant $b \leq B \leq 1/b$ such that $u_{T,i} \in (B - \delta, B+ \delta)$ for all $i$. From (\ref{eq:VLSclosed2}),
we get that
	$\frac{1}{v_{t,j}} \in (B - \delta, B+ \delta)$. Taking $\delta$ sufficiently small, we obtain
	\begin{align}
		u_{i,T} v_{j,T} \in \left( \frac{B - \delta}{B + \delta}, \frac{B + \delta}{B  - \delta}  \right) \subset (1 - b^2\epsilon, 1 + b^2\epsilon),
	\end{align}
	where (\ref{eq:epsilon-delta}) was used in the inclusion step.
	
	Hence,
\begin{align*}
|\alpha_i \beta_j - x_{i,T} y_{j,T}| = |\alpha_i \beta_j ( 1 - u_{i,T}v_{j,T})| \leq |\alpha_i \beta_jb^2\epsilon|\le
\epsilon,
\end{align*}
which completes the proof.

%
%
		
\end{proof}

We now proceed to the Edge Least Square (ELS) algorithm, which is a message passing version of the VLS algorithm. In this algorithm,
the variables, rather than being supported on nodes, namely  $x_{i,t}$ and $y_{j,t}$,
are now supported on edges and will be correspondingly denoted by $x_{i \rightarrow j, t}$ and $y_{j \rightarrow i, t}$.

{\rule{\linewidth}{1pt}}
\vspace{-0.09in}

EDGE LEAST SQUARES (ELS)
\vspace{-0.07in}

{\rule{\linewidth}{1pt}}

\begin{enumerate}
	\item Initialization: $x_{i \rightarrow j,0}, y_{j \rightarrow i,0}$, for all $i,j\in [n]$. Matrix $M$ and graph $\mathcal{G}$.
	\item For $t = 1, \ldots, T$:
		For each $i \in \mathcal{V}_R$ and $j\sim i$ set
		 \begin{align}			x_{i \rightarrow j, t+1} = \arg \min_{x \in \mathbb{R}^r} \sum_{k:{k \sim i},{k \neq j}} |x^Ty_{k \rightarrow i, t} -  M_{ij}|^2.   \label{ELSupdate1} \end{align}
		 For each $j \in \mathcal{V}_C$ and $i\sim j$ set
		\begin{align}		y_{j \rightarrow i, t+1} = \arg \min_{y \in \mathbb{R}^r} \sum_{k: k \sim j, k \neq i} |y^Tx_{k \rightarrow j, t+1} - M_{ij}|^2.	 \label{ELSupdate2} \end{align}
	\item Compute $x_{i,T} = \frac{1}{\Delta(i)} \sum_{j: i \sim j } x_{i \rightarrow j, T}$ and 	$y_{j,T} = \frac{1}{\Delta(j)} \sum_{j: i \sim j } y_{j \rightarrow i, T}$.
	\item Set $X_T=(x_{i,T}, 1\le i\le n), Y_T=(y_{j,T}, 1\le j\le n)$.   Output $\hat{M} = X_T Y_T^T$.			
\end{enumerate}

\rule{\linewidth}{1pt}

Each iteration of the ELS consists of solving $2 |\mathcal{E}|$ least squares problems, so the total computation time per iteration is $O(r^2 \Delta |\mathcal{E}| )$.

For the special case of rank one matrices, it is possible to conduct an analysis on the ELS iterations along the lines of the proof of Theorem \ref{theorem:VLS} for VLS.
Let $\mathcal{H}$ be the dual graph on the directed edges of $\mathcal{G}$. Here $\mathcal{H} = (\mathcal{V}_H, \mathcal{E}_H)$ where $\mathcal{V}_H = \mathcal{V}_{H,R} \cup \mathcal{V}_{H,C} $ and
$\mathcal{V}_{H,R}$ consists of all directed edges $(i, j)$ of $\mathcal{G}$ where $i\in\mathcal{V}_R$ and $j\in\mathcal{V}_C$,
and $\mathcal{V}_{H,S}$ consists of all directed edges $(j, i)$ of $\mathcal{G}$, defined similarly.
Additionally,
 $(i,j) \in \mathcal{V}_{H,R}$ and $(k,l) \in \mathcal{V}_{H,S}$ are neighbors if and only if $j = k$ or $l=i$.
Similarly to (\ref{eq:VLSclosed1}), we can write the corresponding update rules for ELS as follows
\begin{align} \label{eq:ELSclosed1}
		x_{i \rightarrow j, t+1}  = \frac  { \sum_{k:{k \sim i},{k \neq j}}   M_{ik} y_{k \rightarrow i, t} }{\sum_{k:{k \sim i},{k \neq j}} y_{k \rightarrow i}^2} \quad
		\mbox{and} \quad  y_{j \rightarrow i, t+1}  = \frac  { \sum_{k:{j \sim k},{k \neq i}}   M_{kj} x_{k \rightarrow j, t+1} }{\sum_{k:{j \sim k},{k \neq i}} x_{k \rightarrow j}^2}.
\end{align}

Define $u_{i \rightarrow j, t} = \frac{x_{i \rightarrow j, t} }{\alpha_i}$ and $v_{j \rightarrow i, t} = \frac{y_{i \rightarrow j, t} }{\beta_j}$. Then, similarly to  (\ref{eq:VLSclosed2}), we can write the corresponding update rules for ELS as follows.
\begin{align} \label{eq:ELSclosed2}
		u_{i \rightarrow j,t+1} = \sum_{k:{k \sim i},{k \neq j}}  \frac{     y_{k \rightarrow i,t}^2   }{\sum_{l: l \sim i, l \neq j} y_{l \rightarrow i,t}^2} \frac{1}{v_{k \rightarrow i,t}} \quad \mbox{and} \quad
		\frac{1}{v_{j \rightarrow i,t}} = \sum_{k: j \sim k, k \neq i} \frac{   \alpha_i x_{k \rightarrow j,t} }{\sum_{l: j \sim l} \alpha_i x_{l \rightarrow j,t}} u_{k \rightarrow j,t}.
	\end{align}	
Again, as before, letting $u_t = (u_{i \rightarrow j,t}, \ i \sim j)$ we can write $u_{t+1}  = P_t u_t$ for some stochastic matrix $P_t$. The support of $P_t$ is the graph $\mathcal{V}_{H,R}$ where the two vertices are neighbors
if and only if they are distance two neighbors in $\mathcal{H}$. From the above equations, it is apparent that it is possible to prove a result similar to Theorem \ref{theorem:VLS} for ELS.
We state the result below. We omit the proof as it is identical to the proof of Theorem \ref{theorem:VLS}.
\begin{theorem} \label{theorem:ELS}
	Let $M = \alpha \beta^T$ with $\alpha, \beta \in \mathbb{R}^n$ and suppose there exists $0 < b < 1$ such that for all $i,j \in [n]$, we have
	$b \leq \alpha_i, \beta_j \leq 1/b$. Suppose that the graph $\mathcal{H}$ is connected and has diameter $d = c \log n$ for some fixed constant $c$ and maximum degree $\Delta$.
 There exists a constant $\gamma > 0$ which depends on $c,\Delta$ and $b$ only,
 such that for any initialization  $b \leq x_{i \rightarrow j,0}, y_{j \rightarrow i, 0} \leq 1/b, i \in [n]$ and
$\epsilon >0$, there exists an iteration number
	$T = O(n^\gamma \log n)$ such that after $T$ iterates of ELS, we have $\frac{1}{n} \|X_T Y_T^T - M \|_F < \epsilon$.
\end{theorem}


\section{Experiments} \label{sec:sim}
In this section, we provide simulation results for the VLS and ELS algorithms with particular focus on
\begin{itemize}
	\item[(a)] The  convergence rate of VLS vs ELS
	\item[(b)] Success of  VLS and ELS for rank $r\ge 1$.
\end{itemize}

In view of Theorem \ref{theorem:VLS}, we generate $\alpha, \beta \in \mathbb{R}^n$ independently and uniformly
at random from $U[0.01, 0.99]$. We then compare the decay in root mean square error (RMS) defined below in (\ref{rms}) with number of iterations.
To do so, we first generate a uniformly random $3$-regular bipartite graph $\mathcal{G} = (\mathcal{V}, \mathcal{E})$ on $2n$ vertices with $n$ vertices on each side and keep it fixed for the experiment. We then run VLS and ELS on
$M_{\mathcal{E}}$.
 Random regular graphs are known to be connected with high probability, and we did not find significant variation in results by changing the graph. Since ELS requires about a $\Delta$ factor more computation per iteration,
 we plot the decay of RMS vs normalized iterations index which is defined as  $\frac{\mbox{iteration number}}{\mbox{Total iterations}}$ for VLS and $\Delta \frac{\mbox{iteration number}}{\mbox{Total iterations}}$ for ELS.

The root mean square (RMS)  error after $T$ iterations of ether VLS or ELS is defined as
\begin{align} \label{rms}
	\frac{1}{n} \| M - X_T Y_T^T\|_F.
\end{align}

\begin{figure}[h]
\begin{center}
	\includegraphics[scale=0.8]{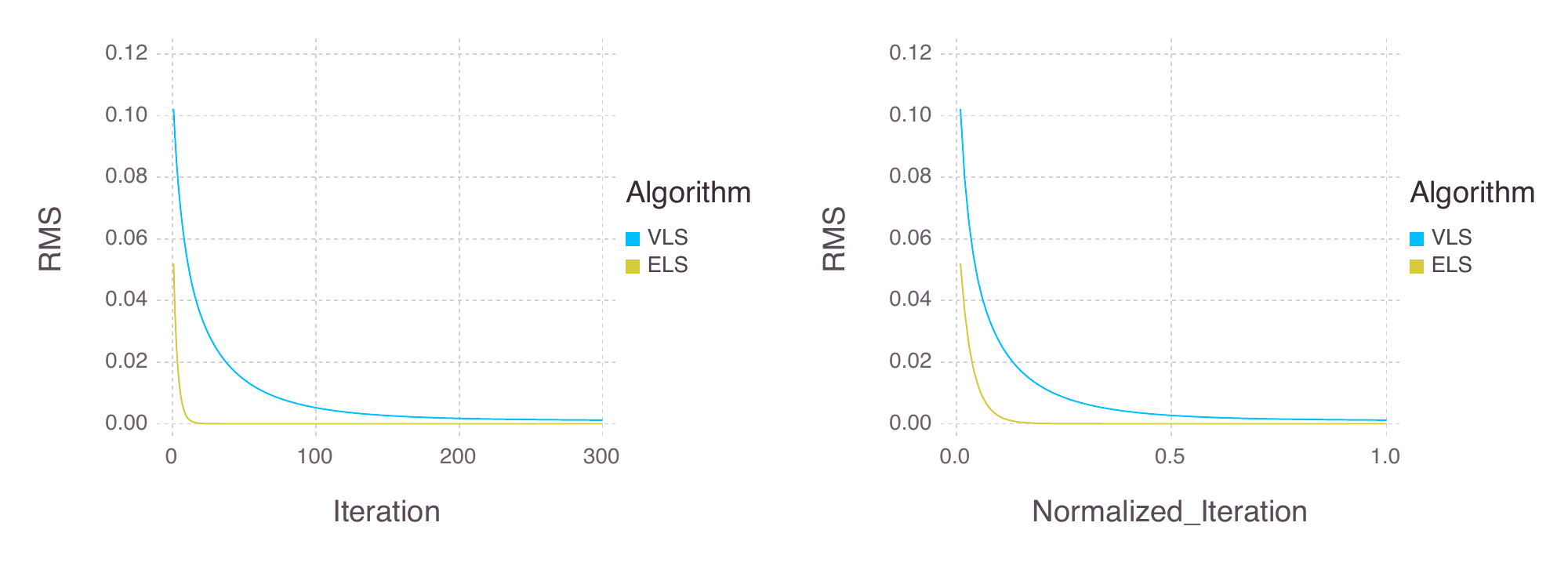}
\end{center}
\caption{RMS vs number of iteration (normalized and un-normalized) for VLS and ELS \label{fig:VLSvsELS}}
\end{figure}

The comparison in Figure~\ref{fig:VLSvsELS} (computed for $n=100$) demonstrates that ELS converges faster than VLS. We find that this effect is even more pronounced when $r > 1$.

To compare VLS and ELS for rank $r$ matrices, we generate
 each entry of $\alpha\in \R^{n\times r}$ and $\beta\in \R^{n\times r}$
 uniformly from the interval $[-1,1]$. We generate a random $r+1$ regular bipartite graph $\mathcal{G} = (\mathcal{V}, \mathcal{E})$
 to ensure that the minimum degree requirement is met. Then we generate another edge set $\mathcal{E}_1$, where each edge exists independently with probability $c/n$.
 Finally we  set
  $\mathcal{E} = \mathcal{E} \cup \mathcal{E}_1$. We plot the empirical fraction of failure obtained from $200$ iterations, where a failure is assumed to occur when the algorithm (VLS or ELS) fails to achieve an RMS less than
  $10^{-3}$ within $500$ iterations. In fact a divergence characterized by an explosion in the RMS value is usually observed at a much earlier iteration whenever there the algorithm fails.
  Figure~\ref{fig:Pfvsc} shows the results for ELS on the left when $r = 2$ and $r=3$ respectively. This provides evidence for the success of ELS even with a cold start.
  On the right of Figure~\ref{fig:Pfvsc} we plot the same for VLS with cold start for $r=2$, showing that it does not always succeed.


%
%

\begin{figure}[h]
\begin{center}
	\includegraphics[scale=0.8]{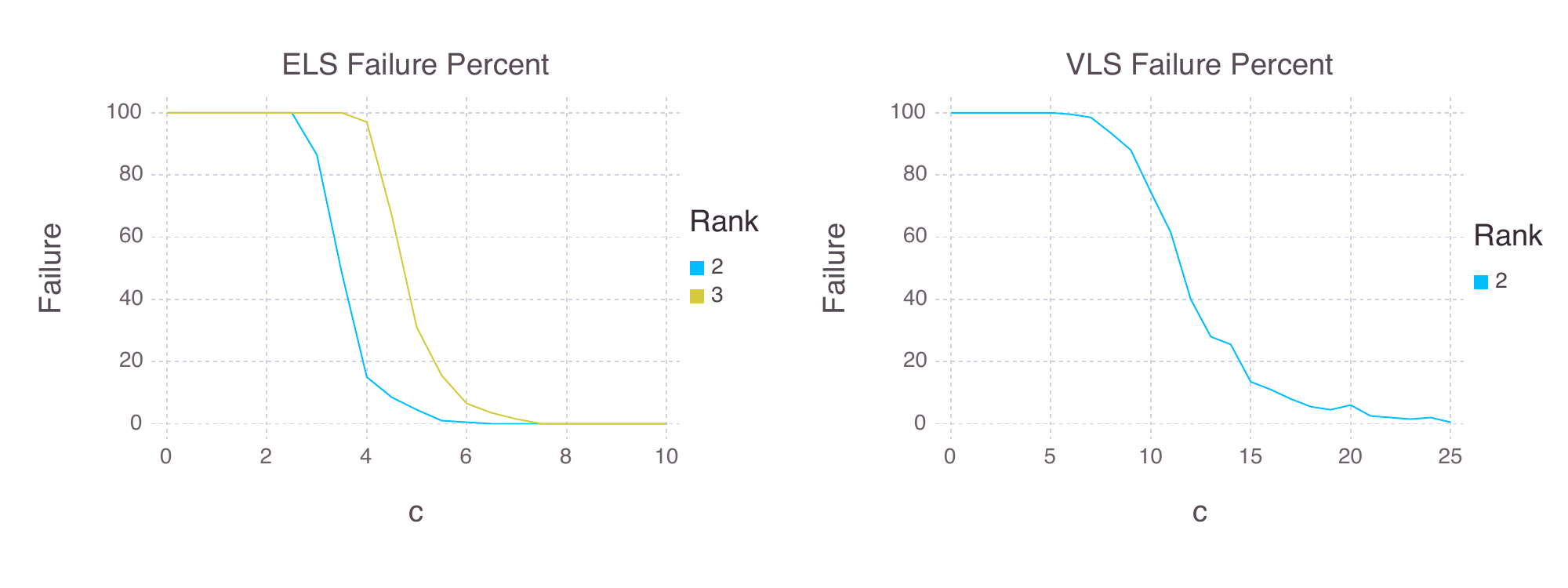}
\end{center}
\caption{{\bf Left:} ELS failure fraction vs $c$ for $r=2$ with planted $3$-regular graph, and $r=3$ with planted $4$-regular graph ($n=100$)  {\bf Right:} VLS failure fraction vs $c$ for $r=2$ with planted $3$-regular graph}
 \label{fig:Pfvsc}
\end{figure}

The figures suggest the emergence of a phase transition. For each algorithm and rank value $r$ there seems to be a critical degree $c^*_{\mathcal{A},r}$ such
that the algorithm $\mathcal{A}$ succeeds with high probability when $c>c^*_{\mathcal{A},r}$ and fails with high probability otherwise.
Furthermore, it appears again based on the simulation results, that the $c^*_{\mathcal{A},r}$ for ELS is smaller than the one for VLS.
In particular, for VLS the threshold appears to be around $3.5$ for ELS when $r=2$, whereas for ELS it appears to be around $12$, for the same value of $r$.
In other words, ELS appears to have a lower sample complexity required for it to succeed. Whether these observations can be theoretically established
is left as an intriguing open problem.

\bibliographystyle{unsrt}
\bibliography{Reference_LR}

\begin{thebibliography}{1}

\bibitem{Meka08}
R.~Meka, P.~Jain, C.~Caramanis, and I.~S. Dhillon.
\newblock Rank minimization via online learning.
\newblock {\em ICML}, pages 656--663, 2008.

\bibitem{candes2009exact}
Emmanuel~J Cand{\`e}s and Benjamin Recht.
\newblock Exact matrix completion via convex optimization.
\newblock {\em Foundations of Computational mathematics}, 9(6):717--772, 2009.

\bibitem{CandesTao09}
E.~Candes and T.~Tao.
\newblock The power of convex relaxation: near optimal matrix completion.
\newblock {\em IEEE Transactions on Information Theory}, 56(5):2053--2080,
  2009.

\bibitem{keshavan2009matrix}
Raghunandan~H Keshavan, Sewoong Oh, and Andrea Montanari.
\newblock Matrix completion from a few entries.
\newblock In {\em Information Theory, 2009. ISIT 2009. IEEE International
  Symposium on}, pages 324--328. IEEE, 2009.

\bibitem{keshavan2009matrix2}
Raghunandan Keshavan, Andrea Montanari, and Sewoong Oh.
\newblock Matrix completion from noisy entries.
\newblock In {\em Advances in Neural Information Processing Systems}, pages
  952--960, 2009.

\bibitem{jain2013low}
Prateek Jain, Praneeth Netrapalli, and Sujay Sanghavi.
\newblock Low-rank matrix completion using alternating minimization.
\newblock In {\em Proceedings of the forty-fifth annual ACM symposium on Theory
  of computing}, pages 665--674. ACM, 2013.

\bibitem{Lenka14}
Y.~Kabashima, F.~Krzakala, M.~M\'{e}zard, A.~Sakata, and L.~Zdeborov\'{a}.
\newblock Phase transitions and sample complexity in bayes-optimal matrix
  factorization.
\newblock {\em arXiv:1402.1298.}, 2014.

\bibitem{Korin09}
Y.~Koren, R.~M. Bell, and C.~Volinsky.
\newblock Matrix factorization techniques for recommender systems.
\newblock {\em IEEE Computer}, 42(8):30--37, 2009.

\bibitem{Korensolo09}
Y.~Koren.
\newblock The {B}ell{K}or solution to the {N}etflix grand prize.
\newblock 2009.

\end{thebibliography}

\end{document}